\documentclass[letterpaper, conference]{ieeeconf}
\usepackage{times}

\IEEEoverridecommandlockouts                              

\usepackage{multicol}
\usepackage[bookmarks=true]{hyperref}
\usepackage{amsmath}
\usepackage{framed}
\usepackage{amssymb}
\usepackage{bm}
\usepackage[version-1-compatibility]{siunitx}

\usepackage{booktabs}
\usepackage{amsthm}
\usepackage{graphicx} 
\usepackage{tikz}

\tikzstyle axes label=[font=\footnotesize, text width=3cm, text centered]

\usepackage{color}
\usepackage{mathrsfs}

\graphicspath{{figures/}}

\theoremstyle{plain} 
 \newtheorem{proposition}{Proposition}
\newtheorem{conjecture}{Conjecture}

\theoremstyle{definition}
\newtheorem{definition2}{Definition}

\theoremstyle{definition}

\newtheorem{remark2}{Implementation remark}

\usepackage[ruled, linesnumbered]{algorithm2e}
\let\oldnl\nl
\newcommand{\nonl}{\renewcommand{\nl}{\let\nl\oldnl}}
\SetKwIF{If}{ElseIf}{Else}{if}{ then}{elif}{else}{}%
\SetKwFor{For}{for}{ do}{}%
\SetKwFor{ForEach}{foreach}{ do}{}%
\SetKwInOut{Input}{Input}%
\SetKwInOut{Output}{Output}%
\AlgoDontDisplayBlockMarkers%
\SetAlgoNoEnd%
\SetAlgoNoLine%
\DontPrintSemicolon

\renewcommand{\vec}[1]{\bm{\mathrm{#1}}}  

\renewcommand{\d}{\mathrm{d}}

\newcommand{\bbR}{\mathbb{R}}

\begin{document}

\title{\LARGE Time-Optimal Path Tracking via Reachability Analysis}

\author{Hung Pham, Quang-Cuong Pham \\
  \thanks{Hung Pham and Quang-Cuong Pham are with Air Traffic
    Management Research Institute (ATMRI) and Singapore Centre for 3D
    Printing (SC3DP), School of Mechanical and Aerospace Engineering,
    Nanyang Technological University, Singapore (email:
    hungpham2511@gmail.com, cuong.pham@normalesup.org)
  }}

\maketitle

\begin{abstract}
  Given a geometric path, the Time-Optimal Path Tracking problem
  consists in finding the control strategy to traverse the path
  time-optimally while regulating tracking errors. A simple yet
  effective approach to this problem is to decompose the controller
  into two components: (i)~a path controller, which modulates the
  parameterization of the desired path in an online manner, yielding a
  reference trajectory; and (ii)~a tracking controller, which takes
  the reference trajectory and outputs joint torques for
  tracking. However, there is one major difficulty: the path controller
  might not find any feasible reference trajectory that can be tracked
  by the tracking controller because of torque bounds.  In turn, this
  results in degraded tracking performances. Here, we propose a new
  path controller that is guaranteed to find feasible reference
  trajectories by accounting for possible future
  perturbations. The main technical tool underlying the proposed
  controller is Reachability Analysis, a new method for analyzing path
  parameterization problems. Simulations show that the proposed
  controller outperforms existing methods.
\end{abstract}

\IEEEpeerreviewmaketitle

\section{Introduction}
\label{sec:intro}

Time-optimal motion planning and control along a predefined path are
fundamental and important problems in robotics, motivated by many
industrial applications, ranging from machining, to cutting, to
welding, to painting, etc.

The \emph{planning} problem is to find the Time-Optimal Path
Parameterization (TOPP) of a path under kinematic and dynamic
bounds. The underlying assumptions are that the robot is perfectly
modeled, no perturbations during execution and no initial tracking
errors. This problem has been extensively studied since the
1980's~\cite{bobrow1985time}, see~\cite{pham2014general,pham2015time}
for recent reviews.

The \emph{control} problem, which looks for a control strategy to
time-optimally track the path while \emph{accounting for model
  inaccuracies, perturbations and initial tracking errors}, is
comparatively less well understood. We refer to this problem as the
Time-Optimal Path Tracking problem, or ``path tracking problem'' in
short.

\subsection{Approaches to Time-Optimal Path Tracking}
\label{sec:related-works}

The first approach to the path tracking problem was proposed by Dahl
and colleagues in the 1990's~\cite{dahl1990torque,
  dahl1994path}. Suppose that we are given a geometric path
$\vec p(s)_{s\in[0,1]}$.  In Dahl's approach, termed Online Scaling
(OS), the path tracking controller is composed of two sub-controllers:
a \emph{path controller} and a \emph{tracking controller}, see
Fig.~\ref{fig:osdiagram}. The path controller generates a \emph{path
  parameterization} $s(t)$ (``scaling'') by controlling the path
acceleration $\ddot s(t)$, from which it returns ``online'' a
reference trajectory $(\vec q_d,\dot{\vec q}_d)$ via the relations
\begin{equation*}
  \begin{aligned}
    \vec q_{d}(t) &= \vec p(s(t)),\;
 \dot{\vec q}_{d}(t) = \vec p'(s(t)) \dot s(t).
  \end{aligned}
\end{equation*}
The tracking controller then takes the reference trajectory
$(\vec q_d,\dot{\vec q}_d)$ and generates the joint
torques $\vec \tau$ to drive the current state to the reference
state. In OS, the tracking controller is usually as a computed-torque
tracking controller with fixed Proportional-Derivative (PD) gains.
Thus, the problem is reduced to designing a path controller that can
regulate the path tracking errors while tracking a reference
parameterization or minimizing execution time.

\tikzstyle block=[font=\normalsize, text width=3cm, text centered]
\tikzstyle signal=[font=\small, text width=3cm, text centered]
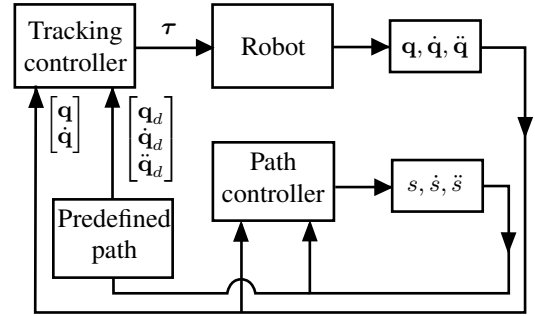
\begin{figure}[t]
  \centering
  \begin{tikzpicture}
    \node[anchor=south west,inner sep=0] (image) at (0,0)
    {\includegraphics[]{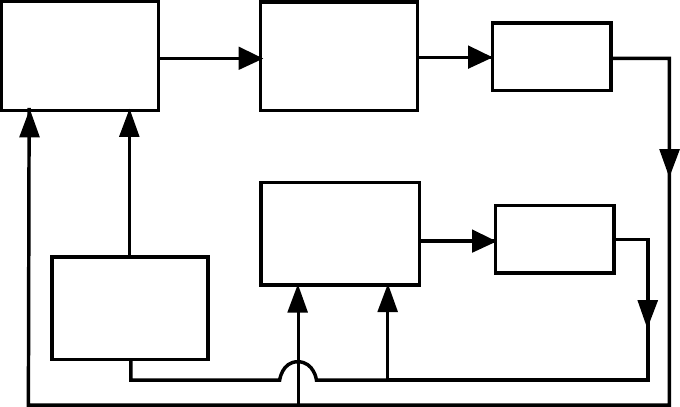}};
    \begin{scope}[x={(image.south east)}, y={(image.north west)}]

 
    \node[block] (robot) at (0.50, 0.86) {Robot};
    \node[block] (pitau) at (0.12, 0.86) {Tracking\\controller};
    \node[block] (piP) at (.50, .44) {Path\\controller};
    \node[block] at (.195, .25) {Predefined\\path};
    \node[signal, anchor=south] at (0.81, 0.795)
         {$\vec q, \dot {\vec q}, \ddot {\vec q}$};
    \node[signal, anchor=south] at (0.3, 0.87)
         {$\vec \tau$};
    \node[signal, anchor=south] at (0.81, 0.35)
         {$s, \dot {s}, \ddot {s}$};
    \node[signal, anchor=north] at (0.265, 0.73) {
      $ \begin{bmatrix}
        \vec q_d \\ \dot {\vec q}_d \\ \ddot {\vec q}_d 
      \end{bmatrix} $
         };

    \node[signal, anchor=north] at (.1,  .73) {
      $ \begin{bmatrix}
        \vec q \\ \dot {\vec q}
      \end{bmatrix} $};

    \end{scope}
  \end{tikzpicture}
  \caption{\label{fig:osdiagram} Block diagram of an Online Scaling
    controller.}
\end{figure}

There have been a number of developments to OS. In~\cite{Arai1994},
the author proposed to use an observer to estimate the online
constraints on the parameterization. In~\cite{Bianco2011}
and~\cite{Gerelli2008}, OS was extended to handle manipulators with
elastic joints or are subject to high-order dynamics such as torque
rate or jerk.

Yet these developments neglect a fundamental problem: there is
\emph{no guarantee} for the path controller to find feasible controls
at execution time. In fact, this issue is recognized in most of the papers
devoted to OS. For example, in the original
paper~\cite{dahl1990torque}, the authors proposed to use the nominal
control if there is no feasible control for the path controller. In a
more recent work~\cite{Bianco2011}, the authors asserted that:
``\emph{since [the path control] bounds are online evaluated [\dots],
  it is not possible to guarantee [\dots that] a feasible solution
  exists [\dots]}''. Yet, employing arbitrary substitute controls when
no feasible control exists will generate large path tracking
errors. For time-optimal path tracking, this issue of \emph{infeasibility} is
far from rare since, by Pontryagin Maximum Principle, time-optimality
is associated with saturating torque bounds at almost every time
instant.

A simpler approach to the path tracking problem can be found
in~\cite{kieffer1997robust}. The authors proposed to consider more
conservative torque bounds at the planning stage, ``reserving''
thereby some torques for tracking during execution. However, this
approach is clearly sub-optimal.

Recently, some authors considered the full optimal control problem,
whose state is $(\vec q, \dot{\vec q}, s, \dot s)$ and control is
$(\vec \tau, \ddot s)$, and applied Nonlinear Model Predictive Control
(NMPC)~\cite{faulwasser2016nonlinear,
  faulwasser2017implementation}. While NMPC can account for hard
constraints on state and control, it has some limitations.  First,
ensuring stability is still
non-trivial~\cite{mayne2000constrained}. For instance,
in~\cite{faulwasser2016nonlinear}, to achieve stability, the path
tracking NMPC controller requires \emph{hand-designed} terminal
sets. Second, the time-optimality objective is challenging since it is
non-convex in the time domain~\cite{verscheure2008practical}.

\subsection{Contribution and organization of the paper}
\label{sec:our-prop-appr}

To guarantee that the path controller will always find feasible
controls requires a certain level of foresight: one needs to take into
account \emph{all possible perturbations} along \emph{the path}. In
this paper, we build on the recent formulation of TOPP by Reachability
Analysis~\cite{Pham2017} to provide such foresight. Specifically, we
compute sets of robust controllable states~\footnote{These are
  parameterization states, which are defined as squared path
  velocities. In Section~\ref{sec:robust-path-accel}, precise
  definitions are given.} that guarantee the existence of feasible
controls for bounded tracking errors. From these sets, a class of path
tracking controllers that have exponential stability and feasiblity
guarantees is identified. The time-optimal controller is then
found straightforwardly.


The rest of the paper is organized as
follows. Section~\ref{sec:path-foll-probl} provides the background on
the path tracking problem and the path tracking
controller. Section~\ref{sec:robust-path-accel} presents the main
contributions. Section~\ref{sec:experimental-results} reports
experimental results, demonstrating the effectiveness of the proposed
approach. Finally, Section~\ref{sec:conclusion} delivers concluding
remarks and sketches directions for future research.

\subsection{Notation}
\label{sec:notational-remark}
We adopt the following conventions.  Vectors are denoted by bold
letters: $\vec x$. The $i$-th component of a vector is denoted using
subscript $i$: $x_{i}$. A vector quantity at stage $j$ is denoted by
bold letters with subscript $j$: $\vec x_{j}$, its $i$-th component is
denoted by adding a second subscript $i$: $x_{ji}$.  Define function
$\phi(s, s^{d}; \vec p):=[\vec p(s)^{\top}, \vec p'(s)^{\top}
s^{d}]^{\top}$, argument $\vec p$ will be neglected if clear from
context. If $\vec x, \vec y$ are two vectors, $(\vec x,\vec y)$ denote
the concatenated vector $(\vec x^{\top},\vec y^{\top})^{\top}$.
Values of differential quantities, such as $\dot {\vec q}$,
have superscript $d$: $\vec q_{0}^{d}$.

\section{Background: Path Tracking problem and controllers}
\label{sec:path-foll-probl}

\subsection{Path Tracking problem}
\label{sec:path-tracking}

Path tracking is the problem of designing a controller to make the
robot's joint positions follow a path parameterization of a predefined
path. The path parameterization is not fixed but is generated by the
controller in an online manner.

Specifically, we consider a $n$-dof manipulator with the dynamic
equation
\begin{equation}
  \label{eq:rigid}
  \vec M(\vec q) \ddot{\vec q} + \dot{\vec q}^{\top}\vec C(\vec q) \dot {\vec q}+ \vec h(\vec q)= \vec {\tau},
\end{equation}
where $\vec q \in \bbR^{n}$ and $\vec \tau \in \bbR^{n}$ denote the
vectors of joint positions and joint torques; $\vec M, \vec C, \vec h$
are appropriate functions. The joint torques are bounded:
\begin{equation}
  \label{eq:t-bound}
  \vec \tau_{\min} \leq \vec \tau \leq \vec \tau_{\max}.
\end{equation}

\emph{A geometric path} is a twice-differentiable function
$\vec p(s)_{s\in [0, 1]}\in \bbR^{n}$. A \emph{path parameterization}
is a twice-differentiable non-decreasing function
$s(t)_{t\in [0, T]} \in [0, 1]$. Path parameterizations are also
subject to terminal velocity constraints of the form
$\dot s(T) \in \mathbb{I}_{\mathrm{end}}$, where $\mathbb{I}_{\mathrm{end}}$ is
called the terminal set. To generate the parameterization, one
directly controls the path acceleration. Let $u$ denote the path
parameterization control: $\ddot s = u$.

The state-space equation of the \emph{coupled system} consisting of
the manipulator and the path parameterization reads
\begin{equation}
  \label{eq:ss}
 \frac{\d}{\d t}
 \begin{bmatrix}
   \vec q \\ \dot{\vec {q}} \\ s \\ \dot s
 \end{bmatrix} =
 \begin{bmatrix}
   \dot {\vec q} \\
   \vec M(\vec q)^{-1}(\vec \tau - \dot{\vec q}^{\top}\vec C(\vec q) \dot {\vec q}- \vec h(\vec q))  \\
   \dot s \\
   u
 \end{bmatrix}.
\end{equation}
Let $\vec y$ denote the state of the coupled system
$[\vec q, \dot {\vec q}, s, \dot s]$, Eq.~(\ref{eq:ss}) can be written
concisely as
\begin{equation*}
 \dot {\vec y} = f(\vec y) + g(\vec y)
 \begin{bmatrix}
   \vec \tau \\ u
 \end{bmatrix}.
\end{equation*}

Consider a control law $[\vec \tau, u] = \alpha(\vec y)$ that always
satisfies the torque bounds, one obtains the autonomous dynamics
\begin{equation}
  \label{eq:auto-ss}
  \dot {\vec y} = f(\vec y) + g(\vec y) \alpha(\vec y) = \hat{f}(\vec y).
\end{equation}
We say that $\vec y(t)_{t\in [0, T]}\in \mathbb{R}^{2n + 2}$ is a
\emph{solution} of Eq.~\eqref{eq:auto-ss} if
\begin{equation*}
 \dot {\vec y}(t) = \hat{f} (\vec y(t)), \; y_{2n + 1}(t) \in [0, 1], \; \forall t \in [0, T].
\end{equation*}
It is a \emph{feasible solution} if additionally,
\begin{align*}
  &y_{2n + 2}(t) \geq 0, \; \forall t \in [0, T], \\
  &y_{2n+1}(T) = 1, y_{2n+2}(T) \in \mathbb{I}_{\mathrm{end}}.
\end{align*}
It is a solution \emph{with initial value} $\vec y_{0}$ if
$\vec y(0) = \vec y_{0}$.

The coupled system is \emph{stable} at
$(s_{0}, s_{0}^{d}) \in [0, 1]\times [0, \infty]$ if for any $R>0$,
there exist $r>0$ such that if
$\|(\vec q_{0}, \vec q_{0}^{d}) - \phi(s_{0}, s_{0}^{d}; \vec p)\|_{2}\leq
r$, the solution $\vec y=:(\vec q, \dot {\vec q}, s, \dot s)$ with
initial value $(\vec q_{0}, \vec q_{0}^{d}, s_{0}, s_{0}^{d})$ exists
and
\[
\|(\vec q(t), \dot {\vec q}(t)) - \phi(s(t), \dot s(t))\|_{2} < R, \; \forall t \in [0, T].
\]
The coupled system is \emph{exponentially stable} at
$(s_{0}, s_{0}^{d}) \in [0, 1]\times [0, \infty]$ if it is stable and
there exist $r_{e} > 0$ such that if
$\|(\vec q_{0}, \vec q_{0}^{d}) - \phi(s_{0}, s_{0}^{d}; \vec p)\|_{2}\leq
r_{e}$, the solution $\vec y=(\vec q, \dot {\vec q}, s, \dot s)$ with
initial value $(\vec q_{0}, \vec q_{0}^{d}, s_{0}, s_{0}^{d})$ exists
and satisfies
\[
  \|(\vec q(t), \dot {\vec q}(t)) - \phi(s(t), \dot s(t))\|_{2} < K e^{-\lambda t}, 0 \leq t \leq T,
\]
for some positive real numbers $K, \lambda$.

\subsection{Path Tracking controllers}
\label{sec:two-degree-freedom}
A path tracking controller consists of a path controller, which
controls the path acceleration to generate a desired joint trajectory
$\vec q_{d}(t):=\vec p(s(t))$, and a tracking controller that controls
the joint torques to track the desired joint trajectory.

A common control objective is to track a predefined reference path
parameterization. Here we consider the time-optimal objective, which
is to traverse the path as fast as possible.

Similar to the paper~\cite{dahl1990torque} and subsequent
developments~\cite{dahl1994path,Bianco2011}, we employ the
computed-torque trajectory tracking scheme for the tracking
controller. This scheme implements the following control law
\begin{equation}
  \label{eq:torque}
  \begin{aligned}
  \vec \tau= &\vec M(\vec q)[\ddot {\vec q}_{d} + \vec K_{p} \vec e
  + \vec K_{d} \dot{\vec e} ]  \\
  &+ \dot{\vec q}^{\top}\vec C(\vec q) \dot {\vec q}+ \vec h(\vec q),
  \end{aligned}
\end{equation}
where $\vec e$ denote the joint positions error vector, defined as
$\vec e := \vec q_{d}(t) - \vec q(t)$, $\vec K_{p}$ and $\vec K_{d}$
are the PD gain matrices.  The vector $(\vec e, \dot{\vec e})$ is
called the tracking error.  The first and second time derivatives of
the desired joint trajectory, which are used in~\eqref{eq:torque}, are
given by
\begin{equation*}
  \begin{aligned}
 \dot{\vec q}_{d}(t) &= \vec p'(s(t)) \dot s(t),\\ \ddot {\vec q}_{d}(t) &= 
\vec p'(s(t)) \ddot s(t) + \vec p''(s(t)) \dot s(t)^{2}.
  \end{aligned}
\end{equation*}

Rearranging Eq.~\eqref{eq:torque}, one obtains a formula for joint
torques:
\begin{equation}
  \label{eq:cnst-origin-2}
  \vec \tau= \hat{\vec a}(\vec y)\ddot s + \hat{\vec b}(\vec y)\dot s^{2} + \hat{\vec c}(\vec y),
\end{equation}
where 
\[
  \begin{aligned}
    \hat{\vec a} (\vec y) &= \vec M(\vec p(s) + \vec e)\vec p'(s), \\
    \hat{\vec b} (\vec y)&= \vec M(\vec p(s) + \vec e)\vec p''(s) +
    \vec p'(s)^{\top} \vec C(\vec p(s) + \vec e) \vec p'(s), \\
    \hat{\vec c} (\vec y)&= \vec M(\vec p(s) + \vec e) [\vec K_{p}
    \vec e + \vec K_{d}\dot{\vec e}] \\ & + 2 \dot{\vec e}^{\top}\vec
    C(\vec p(s) + \vec e)\vec p'(s)\dot s + \vec h(\vec p(s) + \vec
    e).
  \end{aligned}
\]

We observe that if the tracking error is zero, the
coefficients $\hat{\vec a}, \hat{\vec b}, \hat{\vec c}$ depend only on
the path position $s$ and not on the path velocity $\dot s$. Indeed in
this case, the coefficients reduce to
\[
  \begin{aligned}
    \vec a (s) &= \vec M(\vec p(s))\vec p'(s), \\
    \vec b (s)&= \vec M(\vec p(s))\vec p''(s) +
    \vec p'(s)^{\top} \vec C(\vec p(s)) \vec p'(s), \\
    \vec c (s)&= \vec h(\vec p(s)).
  \end{aligned}
\]
We call $\vec a, \vec b, \vec c$ the nominal coefficients.
Additionally, by inspection, we see that the coefficients
$\hat{\vec a}, \hat{\vec b}, \hat{\vec c}$ are continuous with respect
to the tracking error $\vec e, \dot {\vec e}$.

It follows from the definition of continuous functions that for any
pair $(s, s^{d})\in [0, 1]\times [0, \infty]$ and any positive number
$R$ there exists $r > 0$ such that for all $i\in[1\dots n]$
\begin{equation}
  \label{eq:continuity}
  \left\|
  \begin{bmatrix}
  \vec e\\ \dot{\vec e}
  \end{bmatrix}
  \right\|_{2} < r \implies
  \left\|
  \begin{bmatrix}
    \hat{a}_{i}(\vec q, \dot {\vec q}, s, s^{d}) - a_{i}(s)\\
    \hat{b}_{i}(\vec q, \dot {\vec q}, s, s^{d}) - b_{i}(s)\\
    \hat{c}_{i}(\vec q, \dot {\vec q}, s, s^{d}) - c_{i}(s)\\
  \end{bmatrix} \right\|_{2}
  < R.
\end{equation}
This result implies that for tracking errors with sufficiently small
magnitude, the coefficients $\hat{\vec a}, \hat{\vec b}, \hat{\vec c}$
vary around the nominal coefficients $\vec a, \vec b, \vec
c$. Furthermore, since the path velocity can always be assumed to be
bounded, we can strengthen this result: there exists $\bar {r} > 0$
such that Eq.~\eqref{eq:continuity} holds for any pair $(s, s^{d})$.

\subsection{Difficulties with designing path controllers}
\label{sec:trajectory-generator}

The fundamental difficulty with designing path controllers is that the
coefficients $\hat{\vec a}, \hat{\vec b}, \hat{\vec c}$ are only
available online. Hence, it is non-trivial to avoid situations in
which there is no path acceleration that satisfies
Eq.~\eqref{eq:cnst-origin-2}. A consequence of this infeasibility is
that exponential convergence of the actual joint trajectory to the
desired joint trajectory is not guaranteed because of saturating
torque bounds.  This difficulty also renders reaching the terminal set
challenging.

\section{Solving the Time-Optimal Path Tracking problem}
\label{sec:robust-path-accel}

\subsection{Exponential stability with robust feasible control laws}
\label{sec:robust-path-param}

In the last section, it was shown that if tracking errors have small
magnitude, the coefficients of Eq.~\eqref{eq:cnst-origin-2} vary around
the nominal coefficients. Motivating by this observation, we introduce
the notion of \emph{robust feasible} control laws. Note that in this
section, \emph{control laws} refer to control laws for selecting path
parameterization controls $u$, not the coupled control
$[\vec \tau, u]$.

Specifically, a \emph{control law} is a function that computes the
path acceleration from the current path position, the current path
velocity and the coefficients:
\[
  u(t) = \pi (s(t), \dot s(t), \hat{\vec a}(t), \hat{\vec b}(t),\hat{\vec c}(t)).
\]
The control law $\pi$ is \emph{robust feasible} at the pair
$(s_{0}, s_{0}^{d})$ if there exists $R>0$ such that for any set of
coefficients
\begin{align*}
  &\hat{\vec a}(t) := \vec a(s(t)) + \vec \Delta_{a}(t),
  \hat{\vec b}(t) := \vec b(s(t)) + \vec \Delta_{b}(t), \\
  &\hat{\vec c}(t) := \vec b(s(t)) + \vec \Delta_{c}(t),
\end{align*}
where the perturbations
$\vec \Delta_{a}, \vec \Delta_{b}, \vec \Delta_{c}$ are arbitrary
continuous functions satisfying
\begin{equation}
  \label{eq:rbp-constraint}
  \|(\Delta_{a, i}(t), \Delta_{b, i}(t), \Delta_{c, i}(t))\|_{2}  < R, \forall t, i \in [1,\dots,n],
\end{equation}
the generated parameterization $s(t)$ is feasible, that is
\begin{align}
  &s(0) = s_{0}, \dot s(0) = s_{0}^{d},\\
  &\exists T, s(T) = 1, \dot s(t) \in \mathbb{I}_{\mathrm{end}},\\
  &\forall t \in [0, T], \dot s(t) \geq 0\\
  &\forall t \in [0, T], \text{Eq.~\eqref{eq:rbp-constraint} holds.}
\end{align}

The following result shows that a robust feasible control law ensures
exponentially stable path tracking.
\begin{proposition}
  \label{prop:1}
  Consider a path tracking controller with (path acceleration) control
  law $\pi$. If $\pi$ is robust feasible at $(s_{0}, s_{0}^{d})$ then
  the coupled system is exponentially stable at $(s_{0}, s_{0}^{d})$.
\end{proposition}
\begin{proof}
  Let $R$ denote the bound on the magnitude of the perturbations such
  that $\pi$ is robust feasible. Select $\bar r>0$ such that
  Eq.~\eqref{eq:continuity} holds for $R$ being the scalar bound in
  Eq.~\eqref{eq:2} and for all $(s, s^{d})$. Select
  $(\vec q_{0}, \vec q_{0}^{d})$ such that the norm of the initial
  tracking error is less than $\bar r$.

  Suppose we remove the torque bounds, the computed-torque tracking
  controller is exponentially stable. Thus, the tracking
  error converges exponentially to zero and its norm remains smaller
  than $\bar r$.  Again using Eq.~\eqref{eq:continuity}, it follows
  that the norm of the perturbations is always smaller than $R$.

  Since $\pi$ is robust feasible at $(s_{0}, s_{0}^{d})$ for $R$ being
  the upper bound on the magnitude of the perturbations, the resulting
  parameterization $s(t)$ is feasible. It follows that the torque bounds
  are always satisfied. Therefore, the coupled path tracking system is
  exponentially stable at $(s_{0}, s_{0}^{d})$ according to the
  definition given in Section~\ref{sec:path-tracking}.
\end{proof}

Notice that the definition of robust feasible control laws does not
require a specific $R$. In fact, as seen in the proof of
Proposition~\ref{prop:1}, continuity of the coefficients guarantees
the existence of $\bar r$ such that Eq.~\eqref{eq:rbp-constraint}
holds for any value of $R$. It can be observed that $\bar r$ is the
radius of a ball lying inside the region of attraction of the path
tracking controller.

\subsection{Characterizing robust feasible control laws}
\label{sec:reformulation}

We now provide a characterization of robust feasible control laws.
This development follows and extends the analysis of the Time-Optimal
Path Parameterization problem in~\cite{Pham2017}.

Discretize the interval $[0, 1]$ into $N+1$ stages
\begin{equation*}
  0=: s_{0}, s_1, \dots, s_N:= 1.
\end{equation*}
Define the \emph{state} $x_i$ and the \emph{control} $u_i$ as the
squared velocity at $s_{i}$ and the constant acceleration over
$[s_{i}, s_{i+1}]$. One obtains the transition function
\begin{equation}
  \label{eq:6}
    x_{i+1} = f_i(x_i, u_i) := x_{i} + 2 \Delta_i u_i,
\end{equation}
where $\Delta_i := s_{i+1} - s_i$.  We say that $u_i$ ``steers'' $x_i$
to $x_{i+1}$. See~\cite{Pham2017} for a derivation of Eq.~\eqref{eq:6}.

At each stage, there are $n$ pairs of torque bounds. The $j$-th pair
of torque bounds at stage $i$ is
\begin{equation}
  \label{eq:discrete-torque-bounds}
  \tau_{\min, j}\leq \tau_{ij}=\hat a_{ij} u_{i} + \hat b_{ij} x_{i} + \hat c_{ij} \leq \tau_{\max, j}.
\end{equation}
The coefficients $\hat a_{ij}, \hat b_{ij}, \hat c_{ij}$ are assumed
to vary around known nominal coefficients $a_{ij}, b_{ij}, c_{ij}$:
\begin{align}
  \label{eq:4}
  &\hat a_{ij} = a_{ij} + \Delta_{a, ij}, \;
  \hat b_{ij} = b_{ij} + \Delta_{b, ij}, \;
  \hat c_{ij} = c_{ij} + \Delta_{c, ij},  \\
  \label{eq:5}
  &\|(\Delta_{a, ij}, \Delta_{b, ij}, \Delta_{c, ij})\|_{2} \leq R.
\end{align}
The terms $(\Delta_{a, ij}, \Delta_{b, ij}, \Delta_{c, ij})$ are also
called the perturbations. $R$ is a parameter that be tuned to account
for the magnitude of initial tracking errors. See the discussion at
the end of Section~\ref{sec:robust-path-param} for more details.  A
control is feasible if it satisfies the constraints and is robust
feasible if it satisfies all realizations of the constraints.
Finally, the terminal velocity constraint is transformed to
\[x_{N}\in X_f:=\{x: \sqrt x \in \mathbb{I}_{\rm{end}}\}.\] 

We say that a state is \emph{robust controllable} at stage $i$ if
there exists a sequence of robust feasible controls that steers it to
$\mathcal{X}_{f}$. The set of robust controllable states at stage $i$
is called the \emph{$i$-stage robust controllable set}
$\mathcal{K}_i$.

In this discrete reformulation, the control law becomes a function
$\pi$ that maps the stage index, the state and the constraint
coefficients
$(i, x_{i}, \hat{\vec a}_{i}, \hat{\vec b}_{i}, \hat{\vec c}_{i})$ to
a control.  Similarly, a control law is \emph{robust feasible} at
$(i, x_{i})$ if it steers $x_{i}$ to the terminal set from stage $i$
for any realization of the constraints with feasible controls.

We now give a characterization of robust feasible control laws.
\begin{proposition}
  \label{prop:robust}
  For any state in $\mathcal{K}_i $ and any realization of the
  constraints, there exists at least one feasible control that steers
  that state to $\mathcal{K}_{i+1}$.  If at any stage $i$, a
  control law steers states in $\mathcal{K}_{i}$ to
  $\mathcal{K}_{i+1}$, it is robust feasible at all robust
  controllable states at all stages.
\end{proposition}

This characterization of robust feasible control laws is only
useful if one can compute the robust controllable sets. To do so, we
first introduce the notion of the robust one-step set.  Given a target
set $\mathbb{I}\subseteq\mathbb{R}$, the $i$-stage \emph{robust
  one-step set} $\mathcal{Q}_{i}(\mathbb{I})$ is the set of states
such that at each state, there is a robust feasible control that
steers it to $\mathbb{I}$.

\begin{proposition}
  \label{prop:2}
  The $i$-stage robust controllable sets, for $i\in[0,\dots,N]$, can be
  computed recursively by
\begin{equation}
  \label{eq:recur}
  \begin{aligned}
    \mathcal{K}_{N} = X_f, \quad  \mathcal{K}_{i} = \mathcal{Q}_i(\mathcal{K}_{i+1}).
  \end{aligned}
\end{equation}
\end{proposition}
A proof of this statement is omitted due to space
constraints. Interested readers can refer to~\cite{Pham2017} for the
proof of a similar result. We can now give a proof of
Proposition~\ref{prop:robust} below.

\begin{proof}[Proof of Proposition~\ref{prop:robust}]
  Let $x_{i}$ be a state in $\mathcal{K}_{i}$. Since $x_{i}$ is robust
  controllable, there exists a sequence of controls
  $(u_{i},\dots,u_{N-1})$ that are robust feasible and the resulting
  sequence of states $(x_{i+1},\dots,x_{N})$ satisfies
  $x_{N}\in \mathcal{X}_{N}$. Observe that this implies $x_{i+1}$ is
  robust controllable, and hence, $u_{i}$ is a robust feasible
  control that steer $x_{i}$ to $\mathcal{K}_{i+1}$.

  Consider a control law that steers states in $\mathcal{K}_{i}$ for
  any stage $i$ to $\mathcal{K}_{i+1}$.  It is clear that this control
  law steers any robust controllable states to $\mathcal{K}_{N}$.
  Since $\mathcal{K}_{N} = \mathcal{X}_{f}$, see~\eqref{prop:2}, the
  control law is robust feasible.
\end{proof}

A class of convex sets that can be handled quite efficiently is the
class of \emph{Conic-Quadratic} representable (CQr)
sets~\cite{ben2001lectures}.  A set of vectors $\vec \epsilon$ is CQr
if it is defined by finitely many conic-quadratic constraints
\begin{equation*}
  \left\|\vec D_{i}
  \begin{bmatrix} \vec \epsilon \\ \vec \nu \end{bmatrix}
  - \vec d_{i}\right\|_{2} \leq \vec p_{i}^{\top}
  \begin{bmatrix} \vec \epsilon \\ \vec \nu \end{bmatrix}
  - q_{i}, \; i \in [1,\dots,k].
\end{equation*}

\begin{proposition}
  \label{prop:char-robust-feas}
  If $\mathbb{I}$ is an interval, the set of state and robust feasible
  control pairs $(x, u)$ that satisfies
  $x + 2\Delta_{i}u\in \mathbb{I}$ is a CQr. Furthermore,
  $\mathcal{Q}_{i}(\mathbb{I})$ is an interval.
\end{proposition}

Indeed, from Eq.~\eqref{eq:discrete-torque-bounds} and
Eq.~\eqref{eq:4}, the $j$-th joint torque is given by
\begin{equation*}
  \tau_{ij}=
  \begin{bmatrix}
   a_{ij} & b_{ij}  & c_{ij} 
  \end{bmatrix}
  \begin{bmatrix}
    u_{i} \\ x_{i} \\ 1
  \end{bmatrix}
  + 
  \begin{bmatrix}
    \Delta_{a, ij} & \Delta_{b, ij} & \Delta_{c, ij}
  \end{bmatrix}
  \begin{bmatrix}
    u_{i} \\ x_{i} \\ 1
  \end{bmatrix}.
\end{equation*}
Since the norm of the perturbation is bounded, see Eq.~\eqref{eq:5},
one obtains the inequality
\begin{equation}
  \label{eq:1}
  \tau_{ij} \leq 
  \begin{bmatrix}
   a_{ij} & b_{ij}  & c_{ij} 
  \end{bmatrix}
  \begin{bmatrix}
    u_{i} \\ x_{i} \\ 1
  \end{bmatrix}
  + 
  R
  \left\|
  \begin{bmatrix}
    u_{i} \\ x_{i} \\ 1
  \end{bmatrix}
  \right\|_{2}.
\end{equation}
It is clear that if and only if the right-hand side is not greater
than $\tau_{\max, j}$, the pair $(u_{i}, x_{i})$ satisfies all
realizations of this constraint. One obtains the conic-quadratic
constraint
\begin{equation}
  \label{eq:2}
  R \left\|
    \begin{bmatrix}
      1 & 0 \\0 & 1 \\ 0& 0
    \end{bmatrix}
    \begin{bmatrix}
      u_{i} \\ x_{i}
    \end{bmatrix}
    +
    \begin{bmatrix}
      0 \\ 0 \\ 1
    \end{bmatrix}
  \right\|_{2}
  \leq -
  \begin{bmatrix}
    a_{ij} & b_{ij}
  \end{bmatrix}
  \begin{bmatrix}
    u_{i} \\ x_{i}
  \end{bmatrix}
  - c_{i} + \tau_{\max, j}.
\end{equation}
Note that the lower bound can be handled in a similar way. Instead
of finding the upper bound of $\tau_{ij}$, one derives the lower bound
\begin{equation}
  \label{eq:3}
  \tau_{ij} \geq 
  \begin{bmatrix}
   a_{ij} & b_{ij}  & c_{ij} 
  \end{bmatrix}
  \begin{bmatrix}
    u_{i} \\ x_{i} \\ 1
  \end{bmatrix}
  - 
  R
  \left\|
  \begin{bmatrix}
    u_{i} \\ x_{i} \\ 1
  \end{bmatrix}
  \right\|_{2}.
\end{equation}
By requiring the right-hand side to be greater than or equal to
$\tau_{\min, j}$, one obtains another conic-quadratic constraint.

Finally, if $\mathbb{I}$ is an interval, the constraint
$x + 2 \Delta_{i}u\in \mathbb{I}$ is equivalent to two linear
inequalities, which are clearly CQr. $\mathcal{Q}_{i}(\mathbb{I})$
being an interval is a simple corollary.

From Proposition~\ref{prop:char-robust-feas}, one can formulate a pair
of conic-quadratic optimization programs to compute the robust
one-step set for any given target set.  One program maximizes $x$
while one program minimizes.  Computing the robust controllable sets
is then possible with Proposition~\ref{prop:2}.

\subsection{A control law for time-optimal path tracking}
\label{sec:control-law-time}

Proposition~\ref{prop:robust} can be used to identify a class of
control laws that are robust feasible, which, by
Proposition~\ref{prop:1} are exponentially stable. What is then
the control law that realizes the shortest traversal time in this
class? We give the following conjecture.

\begin{conjecture}
  In the class of control laws that for all $i$ steers states in
  $\mathcal{K}_{i}$ to $\mathcal{K}_{i+1}$, the control law that
  always chooses \emph{the greatest feasible controls} is
  time-optimal.
\end{conjecture}

We current do not have a proof of this conjecture. Regardless, in our
experiments, the conjecture is verified by comparing the traversal
time with the duration of the time-optimal path parameterization.

\section{Experimental results}
\label{sec:experimental-results}

We simulated a 6-axis robotic arm and controlled it to track a
geometric path $\vec p(s)_{s\in [0, 1]}$ with zero terminal velocity constraint.
The torques bounds are
\begin{equation*}
  \vec \tau_{\max} =   -\vec \tau_{\min} = [120., 280., 280., 120., 80., 80.] (\SI{}{Nm}).
\end{equation*}
Fig.~\ref{fig:motion} visualizes the swinging motion.  Initially the
robot was at rest and had an initial joint positions error with magnitude
$\SI{0.1}{rad}$.  Forward dynamic computations were performed using
OpenRAVE~\cite{diankov_thesis} and the \texttt{dopri5} solver. We
sampled joint torques at sample time $\SI{1}{ms}$.

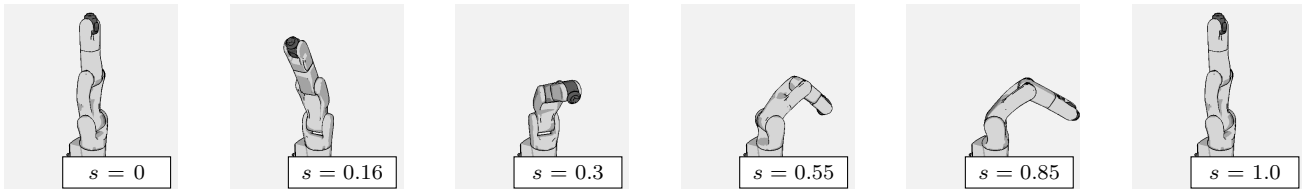
\begin{figure*}
  \centering
  \begin{tikzpicture}
    \node[anchor=south west,inner sep=0] (image) at (0,0)    {\includegraphics[trim={10cm 5cm 10cm 5cm},clip,width=2.3cm]{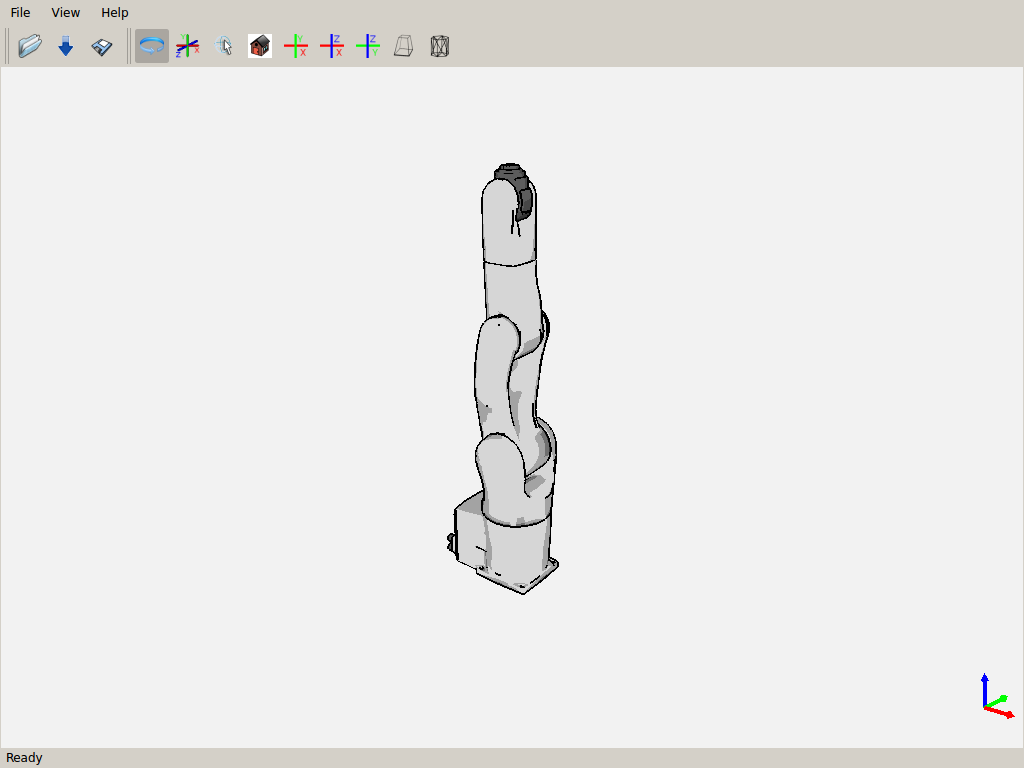}};
    \node[anchor=south west,inner sep=0] (image) at (3cm,0)  {\includegraphics[trim={10cm 5cm 10cm 5cm},clip,width=2.3cm]{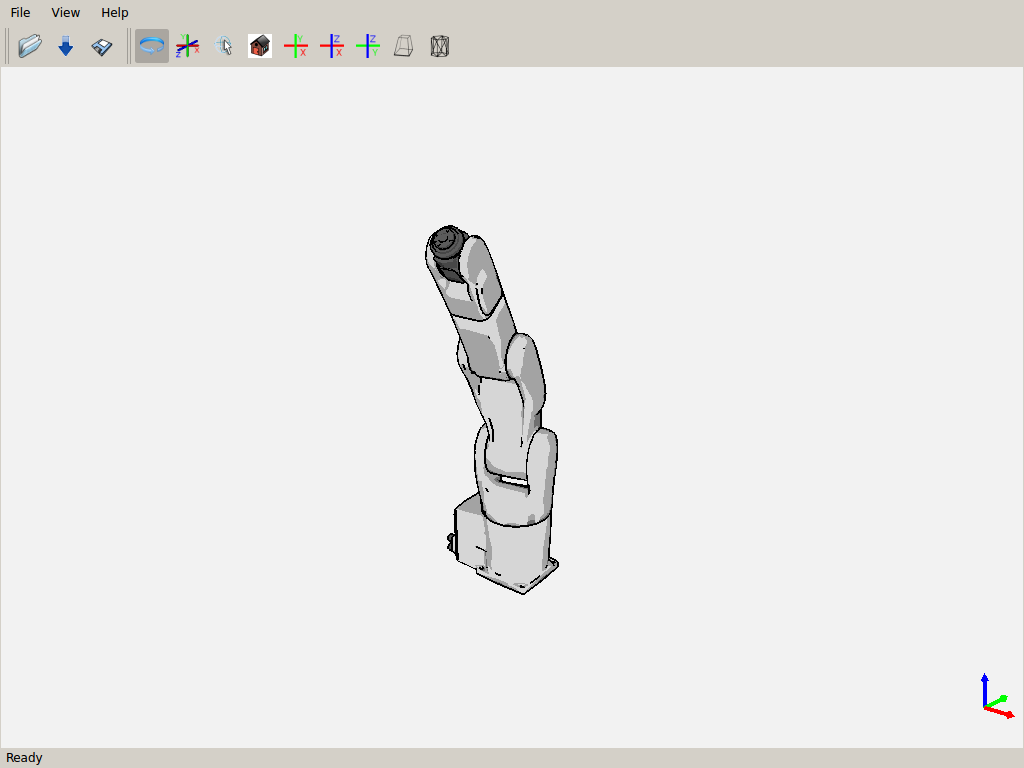}};
    \node[anchor=south west,inner sep=0] (image) at (6cm,0)  {\includegraphics[trim={10cm 5cm 10cm 5cm},clip,width=2.3cm]{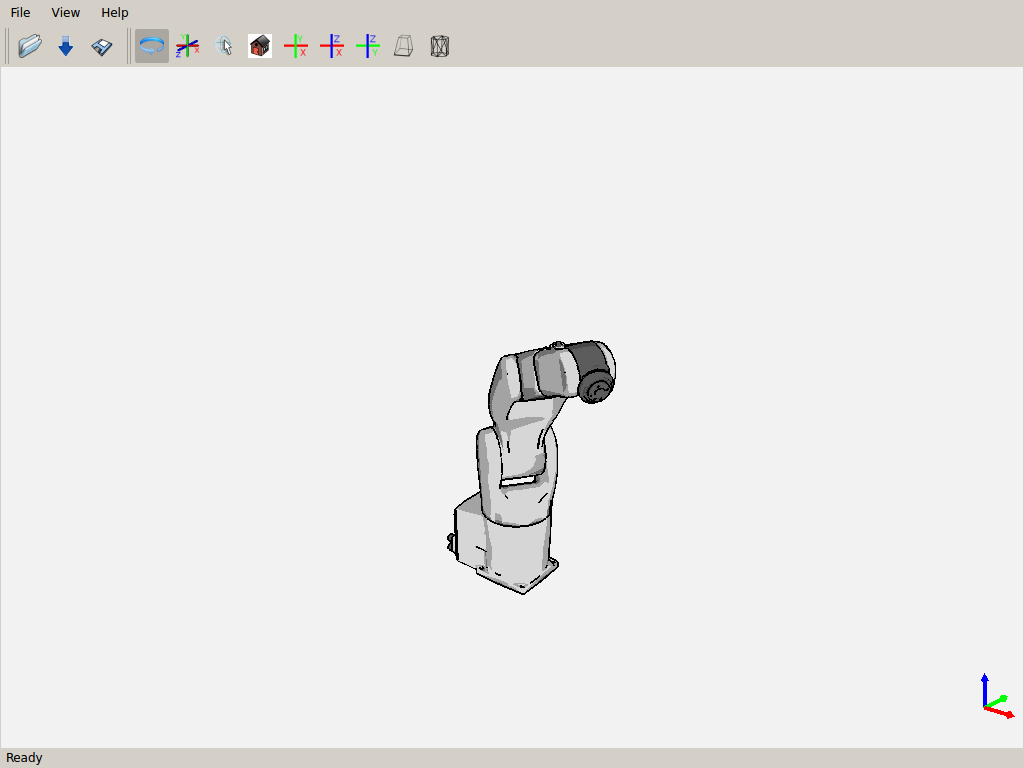}};
    \node[anchor=south west,inner sep=0] (image) at (9cm,0)  {\includegraphics[trim={10cm 5cm 10cm 5cm},clip,width=2.3cm]{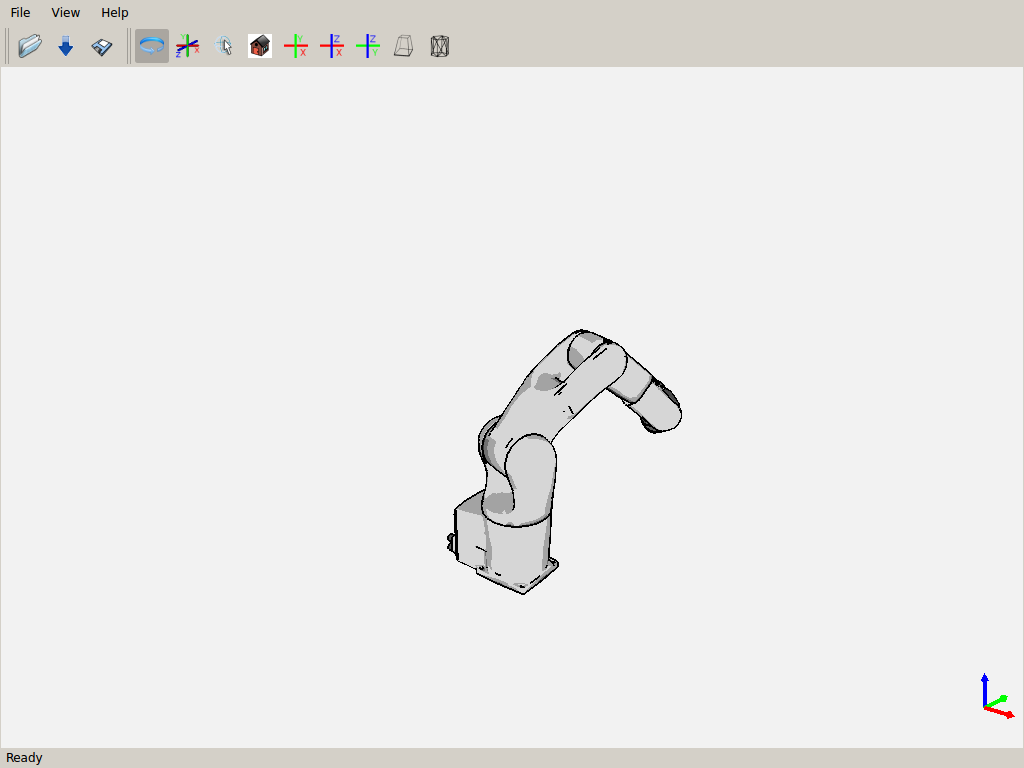}};
    \node[anchor=south west,inner sep=0] (image) at (12cm,0) {\includegraphics[trim={10cm 5cm 10cm 5cm},clip,width=2.3cm]{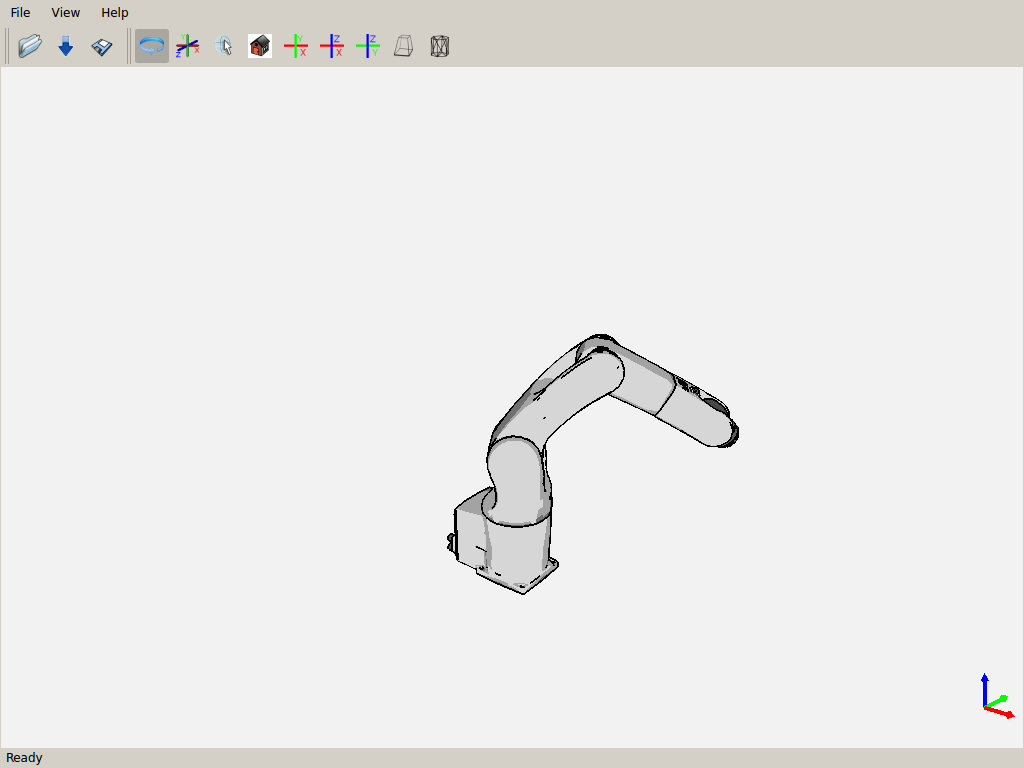}};
    \node[anchor=south west,inner sep=0] (image) at (15cm,0) {\includegraphics[trim={10cm 5cm 10cm 5cm},clip,width=2.3cm]{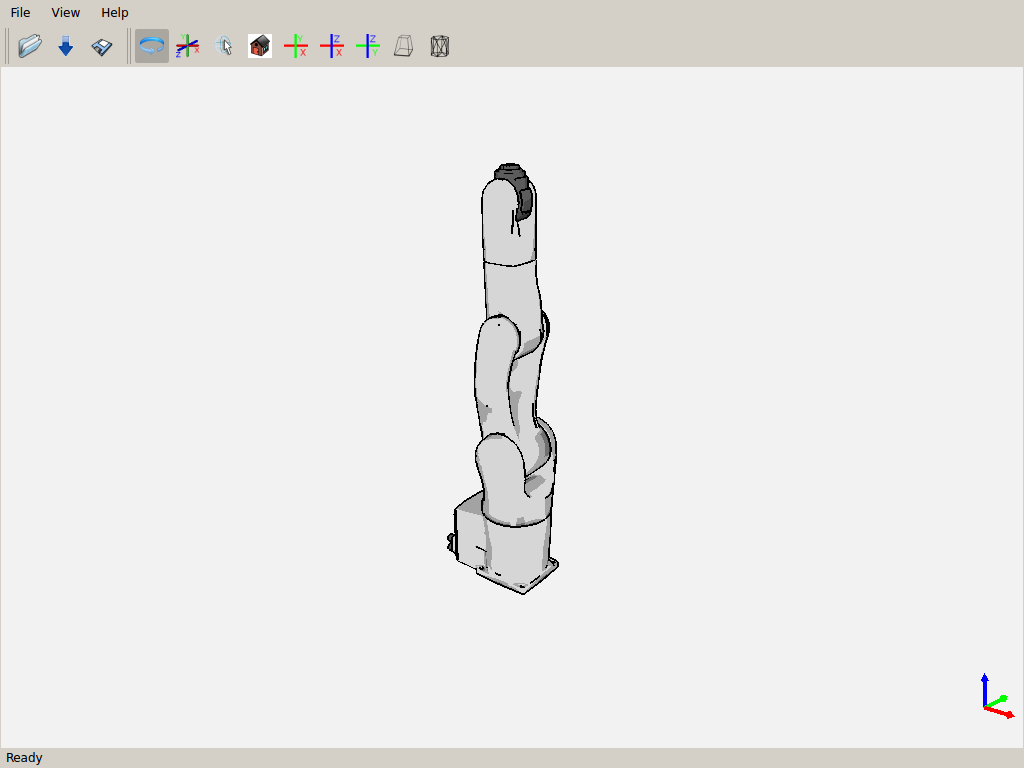}};
    \node[axes label, anchor=south, fill=white, text width=1.2cm, draw] at (1.5cm, 0) {$s=0$};
    \node[axes label, anchor=south, fill=white, text width=1.2cm, draw] at (4.5cm, 0) {$s=0.16$};
    \node[axes label, anchor=south, fill=white, text width=1.2cm, draw] at (7.5cm, 0) {$s=0.3$};
    \node[axes label, anchor=south, fill=white, text width=1.2cm, draw] at (10.5cm, 0) {$s=0.55$};
    \node[axes label, anchor=south, fill=white, text width=1.2cm, draw] at (13.5cm, 0) {$s=0.85$};
    \node[axes label, anchor=south, fill=white, text width=1.2cm, draw] at (16.5cm, 0) {$s=1.0$};
  \end{tikzpicture}
  \caption{\label{fig:motion} The swinging motion used in the
    experiment.}
\end{figure*}

We implemented the time-optimal path tracking controller conjectured
in Section~\ref{sec:control-law-time}, with the bounds on the norm of
the perturbations $R$ set uniformly to $0.5$. The number of
discretization step $N$ was set to $100$.

Computing the robust controllable sets excluding computations of the
coefficients took $\SI{120}{ms}$. We solved the conic-quadratic
programs using the Python interface of
\texttt{ECOS}~\cite{bib:Domahidi2013ecos}. Note that computing the
coefficients involves evaluating the inverse dynamics twice per
stage~\cite{pham2014general}, which had a total running time of
$\SI{40}{ms}$.  Online computations of the controls $(\vec \tau, u)$
took $\SI{0.50}{ms}$ per time step. All computations were done on a
single core of a laptop at $\SI{3.800}{GHz}$.

We compared our controller, called the Time-Optimal Path Tracking
controller (TOPT), with the Online Scaling controller (OS)
in~\cite{dahl1994path} and the Computed-Torque Trajectory Tracking
controller (TT) in~\cite{spong2008robot}. The OS controller tracked
the time-optimal path parameterization of the given geometric path,
while the TT controller tracked the time-optimal \emph{trajectory}.

\begin{table}[htp]
  \caption{Tracking duration and max position errors}
  \label{tab:traj-time}
  \centering
  \begin{tabular}{lrrrr}
    \toprule
           & TOPT & OS & TT \\
    \midrule
    Max pos. err. $(\SI{}{rad})$  & $0.10$ & $0.491$  & $0.493$     \\
    Tracking dur. $(\SI{}{sec})$      & 1.021   & 1.017    & 1.017  \\
    \bottomrule
  \end{tabular}
\end{table}



\begin{figure}[ht]
  \centering
  \begin{tikzpicture}
    \node[anchor=south west,inner sep=0] (imagebottom) at (0,-4.5)
    {\includegraphics[]{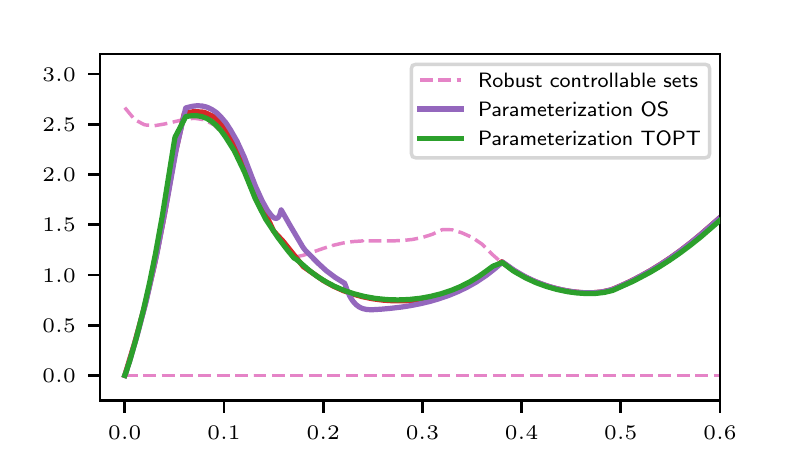}};

    \node[anchor=south west,inner sep=0] (image) at (0,0)
    {\includegraphics[]{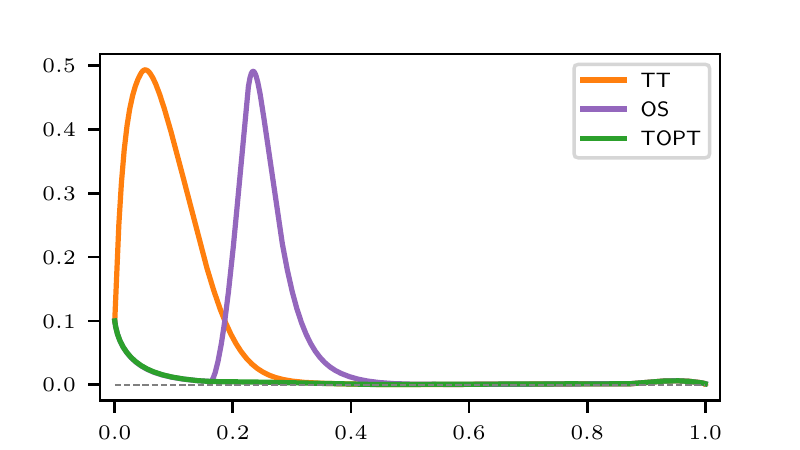}};
    \begin{scope}[x={(image.south east)},y={(image.north west)}]
      \node[axes label, anchor=north] at (0.5, 0.025) {Path position};
      \node[anchor=south] at (0.14, 0.875) {\textbf{A}};

      \node[axes label, anchor=south, rotate=90, text width=5cm]
      at (0.06, 0.5) {Norm of position errors $(\SI{}{rad})$};
    \end{scope}
    \begin{scope}[shift={(imagebottom.south west)},
                  x={(imagebottom.south east)},
                  y={(imagebottom.north west)}]
      \node[axes label, anchor=north] at (0.5, 0.025) {Path position};
      \node[anchor=south] at (0.14, 0.875) {\textbf{B}};
      \node[axes label, anchor=south, rotate=90, text width=5cm] at (0.06, 0.5)
          {Squared path velocity};
    \end{scope}
   
  \end{tikzpicture}
  \caption{\label{fig:comparetraj} \textbf{A}: Norms of joint
    position errors of three controllers: TOPT, OS and
    TT. \textbf{B}: The path position-squared path velocity space
    showing parameterizations (solid lines) generated online by the
    TOPT and OS controllers, the robust controllable sets (dashed
    lines).  }
\end{figure}
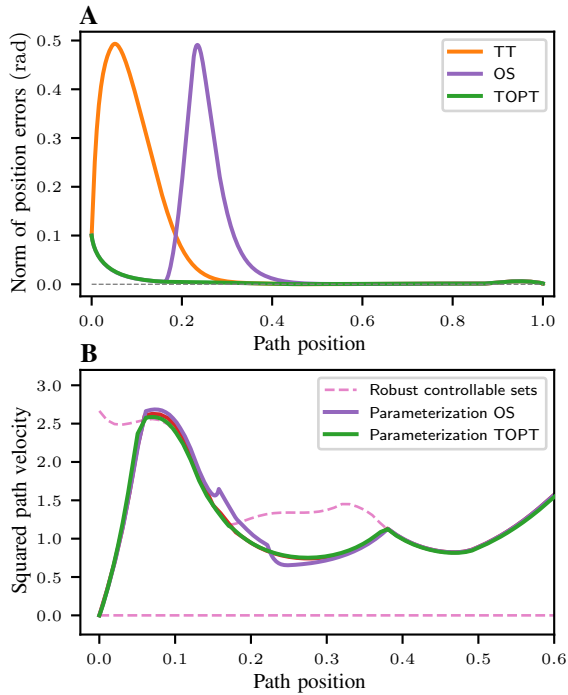

We observe that the TT controller was incapable of handling the
initial position error: Fig.~\ref{fig:comparetraj}A shows that position
errors increased quickly reaching a maximum norm of $\SI{0.49}{rad}$
before stabilizing.  

The OS controller was only able to regulate position errors during the
initial segment. At $s\approx 0.18$, position errors increased
sharply. See Fig.~\ref{fig:comparetraj}A. We note that at this
instance, the OS controller was not able to find any \emph{feasible}
path acceleration.  This event can be observed in Fig.~\ref{fig:comparetraj}B 
as a sharp spike on the generated parameterization.

The TOPT controller did not show any of the above problems. The joint
positions converged quickly to zero. The total tracking duration of
the TOPT controller was slightly higher than the \emph{optimal duration}:
about 1\% longer. See
Table~\ref{tab:traj-time} for the durations.

Finally, we observed that the parameterization generated by the TOPT
controller differed from the parameterization generated by the OS
controller mostly \emph{during decelerating path segments}.  See for
instance the path position interval $s \in [0.05, 0.15]$ in
Fig.~\ref{fig:comparetraj}. Specifically, it can be seen that the TOPT
controller ``slowed down'' in order to stay within the robust
controllable sets. This helped the TOPT controller avoids the
infeasibility at $s=0.18$.

\section{Conclusion}
\label{sec:conclusion}

In this paper, we considered the Time-Optimal Path Tracking problem:
given a geometric path, find the control strategy to traverse the path
time-optimally while regulating tracking errors.  We have introduced
the Time-Optimal Path Tracking controller and shown that the
controller outperforms existing methods.  The key innovation is the
use of robust controllable sets, which intuitively define the sets of
``safe'' path parameterizations that can be tracked while accounting
for possible variations of the coefficients. The technique used in
this paper is Reachability Analysis, a new method for analyzing path
parameterization problems~\cite{Pham2017}.

Several matters have been left for future investigations. Important
questions include how to evaluate and optimize the region of
attraction of path tracking controllers.  Another direction is
extending the approach to handle industrial manipulators with position
or velocity interfaces and to account for higher-ordered constraints
such as joint jerk bounds.

\subsection*{Acknowledgment}

This work was partially supported by grant ATMRI:2014-R6-PHAM (awarded
by NTU and the Civil Aviation Authority of Singapore) and by the
Medium-Sized Centre funding scheme (awarded by the National Research
Foundation, Prime Minister's Office, Singapore).

\bibliographystyle{IEEEtran}
\bibliography{library}

\end{document}